\documentclass{article}
\usepackage{arxiv}

\usepackage{amssymb}
\usepackage{amsthm}
\usepackage{enumerate}
\usepackage{url}
\usepackage{multirow}
\usepackage{float}

\usepackage{mathtools}
\usepackage{algorithm}
\usepackage{algorithmic}
\usepackage[utf8]{inputenc} 
\usepackage[T1]{fontenc}    
\usepackage{url}            
\usepackage{booktabs}       
\usepackage{amsfonts}       
\usepackage{nicefrac}       
\usepackage{microtype}      
\usepackage{lipsum}
\usepackage{graphicx}

\usepackage{csquotes}

\usepackage{xspace}
\usepackage{caption}
\usepackage{subcaption}

\providecommand{\abs}[1]{\lvert#1\rvert}

\DeclareMathOperator*{\argmax}{arg\,max}
\DeclareMathOperator*{\argmin}{arg\,min}

\newtheorem{definition}{Definition}
\newtheorem{lemma}{Lemma}

\usepackage[sorting=none]{biblatex}
\title{A Topological Data Analysis Based Classifier}

\author{
 Rolando Kindelan \\
  Computer Science Department\\ 
  Faculty of Mathematical and Physical Sciences\\
  University of Chile\\ 
  851 Beauchef Av. Santiago de Chile, Chile. \\  
  Center of Medical Biophysics,\\
  Universidad de Oriente, Santiago de Cuba, Cuba.\\
  \texttt{rkindela@dcc.uchile.cl} \\
  \And 
 Jos\'e Fr\'ias \\
  Center for Research in Mathematics\\
  Jalisco S/N, Col. Valenciana, CP: 36023 Guanajuato, Gto., México.\\
  \texttt{frias4@cimat.mx} \\
  \And 
 Mauricio Cerda \\
  Integrative Biology Program\\
  Institute of Biomedical Sciences\\ 
  Biomedical Neuroscience Institute\\
  Center for Medical Informatics and Telemedicine\\
  Faculty of Medicine\\ 
  Universidad de Chile\\
  1027 Independencia Av., Santiago, Chile.\\
  \texttt{mauricio.cerda@uchile.cl} \\
  \And
 Nancy Hitschfeld \\
  Computer Science Department\\ 
  Faculty of Mathematical and Physical Sciences\\
  University of Chile\\ 
  851 Beauchef Av. Santiago de Chile, Chile \\
  \texttt{nancy@dcc.uchile.cl} \\
}

\begin{document}


\maketitle
\begin{abstract}
Topological Data Analysis is an emergent field that aims to discover the underlying dataset's topological information. Topological Data Analysis tools have been commonly used to create filters and topological descriptors to improve Machine Learning (ML) methods. This paper proposes a different Topological Data Analysis pipeline to classify balanced and imbalanced multi-class datasets without additional ML methods. 
Our proposed method was designed to solve multi-class problems. It resolves multi-class imbalanced classification problems with no data resampling preprocessing stage. The proposed Topological Data Analysis-based classifier builds a filtered simplicial complex on the dataset representing high-order data relationships. Following the assumption that a meaningful sub-complex exists in the filtration that approximates the data topology, we apply Persistent Homology to guide the selection of that sub-complex by considering detected topological features. We use each unlabeled point's link and star operators to provide different sized and multi-dimensional neighborhoods to propagate labels from labeled to unlabeled points. The labeling function depends on the filtration entire history of the filtered simplicial complex and is encoded within the persistent diagrams at various dimensions. We select eight datasets with different dimensions, degrees of class overlap, and imbalanced samples per class. The TDABC outperforms all baseline methods classifying multi-class imbalanced data with high imbalanced ratios and data with overlapped classes. Also, on average, the proposed method was better than KNN and weighted-KNN and behaved competitively with SVM and Random Forest baseline classifiers in balanced datasets.
  
\end{abstract}  


\section{Introduction}
\label{sec:intro}

Classification is a Machine Learning (ML) task that employs known data labels to label data with an unknown category. Classification faces challenges such as high dimensionality, noise, and imbalanced data distributions. Topological Data Analysis (TDA) has been successful in reducing dimensionality, and it has demonstrated robustness to noise by inferring the underlying dataset's topology~\cite{Carlsson:Bulletin,EdelsbrunnerHarer2010}. However, how TDA can classify imbalanced datasets has not been explored.

This work proposes a method entirely based on TDA to classify imbalanced and noisy datasets. The fundamental idea is to provide multidimensional and multi-size neighborhoods around each unlabeled point. We use topological invariants computed through PH to detect appropriate neighborhoods. We use the neighborhoods to propagate labels from labeled to unlabeled points. A preliminary version of this work is available in~\cite{kindelan2021classification}.

PH is a powerful tool in TDA that captures the topological features in a nested family of simplicial complexes built on data, according to an incremental threshold value~\cite{EdelsbrunnerHarer2010}. These topological features are encoded, considering those values when they appear (born) and disappear or merge (died). The numerical difference between birth and death scales is called a topological feature's persistence. The evolution of the simplicial structure is encoded using high-level representations called barcodes and persistence diagrams~\cite{EdelsbrunnerHarer2010}. \par
 
Regarding the relation between PH and the classification problem, the existing approach considers hybrid TDA+ML methods, which combine topological descriptors with a conventional ML classifier. Topological descriptors are commonly built based on vectorized or summarized persistence diagrams and barcodes~\cite{ATIENZA2020107509}. Examples of these hybrid TDA+ML methods are: TDA+SVM for image classification ~\cite{tdaretina}, and TDA+k-NN, TDA+CNN and TDA+SVM for time series classification \cite{VenkataramanRT16, umeda2017, timeserie_tda_2016}. Self-Organized Maps were combined with PH tools to cluster and classify time series in the financial domain~\cite{MAJUMDAR2020113868}. 
\par

Approaches to address the imbalanced classification problem can be mainly classified into data level (also known as resampling) and algorithmic level approaches~\cite{Fernandez2018}.  
Data resampling methods balance data by augmenting or removing samples from minority or majority classes, respectively. The Synthetic Minority Oversampling Technique (SMOTE) is the conventional geometric approach to balance classes by oversampling the minority class~\cite{smote}. Multiple variations of SMOTE have been developed~\cite{smotesurvey}, including novel approaches such as the SMOTE-LOF, which takes into account the Local Outlier Factor~\cite{ASNIAR2021} to identify noisy synthetic samples. Furthermore, overlap samples from different classes have been a big issue in imbalance problems. Neighborhood under-sampling from the majority class on the overlapped region has been applied to achieve better results~\cite{Vuttipittayamongkol2020}. These heuristics are simple and can be combined with any classifier as they modify the training set, although they assume data points can always be discarded or generated.
In contrast, SVM, Random Forest, or neural networks adaptations modify their objective function to give higher relative importance to minority class samples~\cite{Ibrahim_2021}. More related to our proposed work, Zhang et al. ~\cite{Zhang2017} propose the Rare-class Nearest Neighbour (KRNN), which defines a dynamic neighborhood based on the inclusion of at least k positive samples.\par

Undersampling and oversampling techniques can be applied as a preprocessing stage of any classifier; however, it could be tricky to devise a winner resampling approach to apply in a multi-class imbalanced data classification problem. Having classes $A, B$, and $C$, class $A$ can be a majority class regarding class $B$, and $B$ can be, at the same time, a majority class concerning class $C$, making it hard to resample $B$. This scenario is known as the problem of multi-minority and multi-majority classes \cite{Fernandez2018}. A typical approach is to split the problem into multiple binary-imbalanced data classification subproblems and perform resampling per subproblem, but other issues may arise. Addressing the multi-class imbalanced classification problem has the advantage of considering the relationships between classes. On the contrary, we can lose information due to the binarization~\cite{Fernandez2018}. 

According to \cite[pp 253]{Fernandez2018}, the imbalance ratio is not the only cause of performance degradation on imbalance problems. Another big concern is the intrinsic data's complexity. When samples from different classes are linearly separable standard classifiers may behave well. Furthermore, state-of-the-art methods may fail where noise and overlapped classes complicate the classification. Gaining knowledge concerning data complexities may help create successful approaches to deal with imbalanced data classification problems. In this context, TDA has proven to be a well-established tool for understanding the data topology that can be applied to unravel the intrinsic complexities of data.  \par 
The work presented in our paper focuses on applying TDA directly to imbalanced data classification without data resampling. The same method is applied to balanced, binary, and multi-imbalanced data. Specifically, the main contributions of this paper are the following:

\begin{itemize}
\item A link-based label propagation method for filtered simplicial complexes.
\item A labeling function depends on the whole filtration history encapsulated within the persistent diagrams at various dimensions. 
\item Filtration values as local weights to ponder label contributions. 
\item PH is a key component for selecting a sub-complex from a filtered simplicial complex.
\item We propose a novel TDA approach to solving classification problems showing advantages for imbalanced datasets without further ML stages.
\end{itemize}

We conduct experiments in eight datasets covering different aspects like class overlapping (non easily separable classes), multiple imbalanced ratios, and high dimensions (more dimensions than samples). We compare the proposed method with k-NN, weighted-KNN, Linear SVM, and Random Forest baseline classifiers. The k-NN algorithm is one of the most popular supervised classification methods used in the basement of SMOTE techniques. The second baseline method is an enhanced version of k-NN, the weighted k-NN (wk-NN) especially suited for imbalanced datasets. Linear SVM and Random Forest, two popular classifiers, are also applied as algorithmic approaches to deal with imbalanced data. This document is organized into several sections. The fundamental concepts and mathematical foundations used in this work are presented in Section~\ref{scc:methods}. Section~\ref{scc:classifier} explains the concepts, algorithms, and methodology of the proposed classification method. Next, Section~\ref{scc:results} describes the experimental protocol to assess the proposed and baseline algorithms, and an analysis of results and the proposed solution is provided. Conclusions are presented in Section~\ref{scc:conclusions}.
\section{Fundamental concepts}\label{scc:methods}

In this section, we introduce mathematical definitions to explain our proposed method; for a complete theoretical basis see~\cite{EdelsbrunnerHarer2010}. 

\subsection{Simplicial Complexes} \label{scc:math}

Simplicial complexes are combinatorial and algebraic objects, which can be used to represent a discrete space encoding topological features of the data space. Concepts related to simplicial complexes are defined briefly as follows: a q-simplex $ \sigma$ is the convex hull of $q + 1$ affinely independent points $\{s_0, \ldots, s_q\}\subset\mathbb{R}^n$, $q\leq n$. The set $\mathcal{V}(\sigma)$ is called the set of \emph{vertices} of $\sigma$ and the simplex $\sigma$ is generated by $\mathcal{V}(\sigma)$; this relation will be denoted by $\sigma=[s_{0},\dots,s_{q}]$. A q-simplex $\sigma$ has dimension $dim(\sigma)=q$ and it has $\abs{\mathcal{V}(\sigma)} = q + 1$ vertices. Given a q-simplex $\sigma$, a d-simplex $\tau$ with $0\leq d \leq q$ and $\mathcal{V}(\tau)\subseteq \mathcal{V}(\sigma)$; $\tau$ is called a \emph{d-face} of $\sigma$, denoted by $\tau \leq \sigma$,  and $\sigma$ is called a \emph{q-coface} of $\tau$, denoted by $\sigma \geq \tau$.  Note that the 0-faces of a q-simplex $\sigma$ are the points in $\mathcal{V}(\sigma)$, the 1-faces are line segments with endpoints in $\mathcal{V}(\sigma)$ and so forth.  A q-simplex has $\binom{q+1}{d+1}$ d-faces and $\sum_{d=0}^{q}\binom{q+1}{d+1} = 2^{q+1}-1$ faces in total. \par 
In order to define homology groups of topological spaces, the notion of simplicial complexes is central:

\begin{definition} [Simplicial complex]\label{def:complex}
A simplicial complex $\mathcal{K}$ in $\mathbb{R}^{n}$ is a finite collection of simplices in $\mathbb{R}^{n}$ such that:

\begin{itemize}
    \item $\sigma \in \mathcal{K} \mbox{ and } \tau \leq \sigma \implies \tau \in \mathcal{K} $.
    \item $\sigma_1, \sigma_2 \in \mathcal{K}\implies \sigma_1 \cap \sigma_2$ is a either a face of both $\sigma_1$ and $\sigma_2$ or empty.        
\end{itemize}
The dimension of $\mathcal{K}$ is $dim(\mathcal{K}) = max\{dim(\sigma) \mid \sigma \in \mathcal{K}\}$. The set $\mathcal{V}(\mathcal{K}) =\cup_{\sigma \in K} \mathcal{V}(\sigma)$
is called the set of \emph{vertices} of $\mathcal{K}$.\par
\end{definition}



\begin{definition}[Star, Closure, Closed Star, and Link]\label{def:starlink}
   Let $\mathcal{K}$ be a simplicial complex, and $\sigma \in \mathcal{K}$ be a q-simplex. The $star$ of $\sigma$ in $\mathcal{K}$ is the set of all co-faces of $\sigma$ in $\mathcal{K}$~\cite{EdelsbrunnerHarer2010}:
   \begin{equation}
      St_\mathcal{K}(\sigma) = \{\tau \in \mathcal{K} \mid \sigma \leq \tau\}. 
   \end{equation}
        
   If $K$ is a subset of simplices $K \subset \mathcal{K}$. The $closure$ of $K$ in $\mathcal{K}$ is the smallest simplicial complex containing $K$: 
   \begin{equation}
      Cl_\mathcal{K}(K) = \{\mu \in \mathcal{K} \mid \mu \leq \sigma \text{ for some } \sigma \in K \}. 
   \end{equation}
   The smallest simplicial complex that contains $St_\mathcal{K}(\sigma)$ is the \emph{closed star} (closure of star) of $\sigma$ in $\mathcal{K}$:   
   \begin{equation}
     \overline{St}_\mathcal{K}(\sigma) = Cl_\mathcal{K}(St_\mathcal{K}(\sigma)).  
   \end{equation}
   
    The $link$ of $\sigma$ is the set of simplices in its closed star that do not share any face with $\sigma$~\cite{EdelsbrunnerHarer2010}:
    \begin{equation}
      Lk_\mathcal{K}(\sigma) = \{\tau \in \overline{St}_\mathcal{K}(\sigma) \mid \tau \cap \sigma = \emptyset\}.  
    \end{equation}        
\end{definition}

Since the link operator concept will be important throughout this paper, we present two equivalent characterizations of this set:

\begin{lemma} \label{lemma:002} Let $\mathcal{K}$ be a simplicial complex and $\sigma\in \mathcal{K}$. Then $Lk_\mathcal{K}(\sigma)$ coincides with the sets
\begin{equation}
 A  = \overline{St}_\mathcal{K}(\sigma) \setminus (St_\mathcal{K}(\sigma) \cup Cl_\mathcal{K}(\sigma)), \text{ and}
 \end{equation}
\begin{equation}\label{eq:linklemma}
B=\bigcup_{\mu \in St_{\mathcal{K}}(\sigma)} \{[\mathcal{V}(\mu) \setminus \mathcal{V}(\sigma)]\}
\end{equation}
\end{lemma}

\begin{proof}
Let $\tau$ be a simplex in $Lk_\mathcal{K}(\sigma)$. In particular, $\tau$ does not belong to $St_\mathcal{K}(\sigma)$ nor $ Cl_\mathcal{K}(\sigma) $ since every simplex in one of these two sets necessarily intersects $\sigma$, then $Lk_\mathcal{K}(\sigma)\subset A$. \\
If $\tau$ is a simplex in $A$, then there exists $\mu\in St_\mathcal{K}(\sigma)$ such that $\tau\leq \mu$ and $(\mathcal{V}(\mu)\setminus\mathcal{V}(\tau))\subset\mathcal{V}(\sigma)$. It follows that 
$\tau=[\mathcal{V}(\mu)\setminus\mathcal{V}(\sigma)]$ and $A\subset B$. Finally, if $\tau\in B$, then $\tau=[\mathcal{V}(\mu)\setminus\mathcal{V}(\sigma)]$ for some $\mu\in St_\mathcal{K}(\sigma)$. It follows that $\tau\in \overline{St}_\mathcal{K}(\sigma)$, but $\tau\cap \sigma =\emptyset$. Then, $B\subset Lk_\mathcal{K}(\sigma)$, and the equivalence of sets A and B is stated.  

\end{proof}

Figure~\ref{fig:linkstar} presents an example of the star and link of the $0$-simplex $[s_4]$ in a given simplicial complex $\mathcal{K}$ built on a point set $S=\{s_{2},s_{3},s_{4},s_{5}\}$. 
    
\begin{figure}
 \centering   
  \includegraphics[width=0.8\columnwidth]{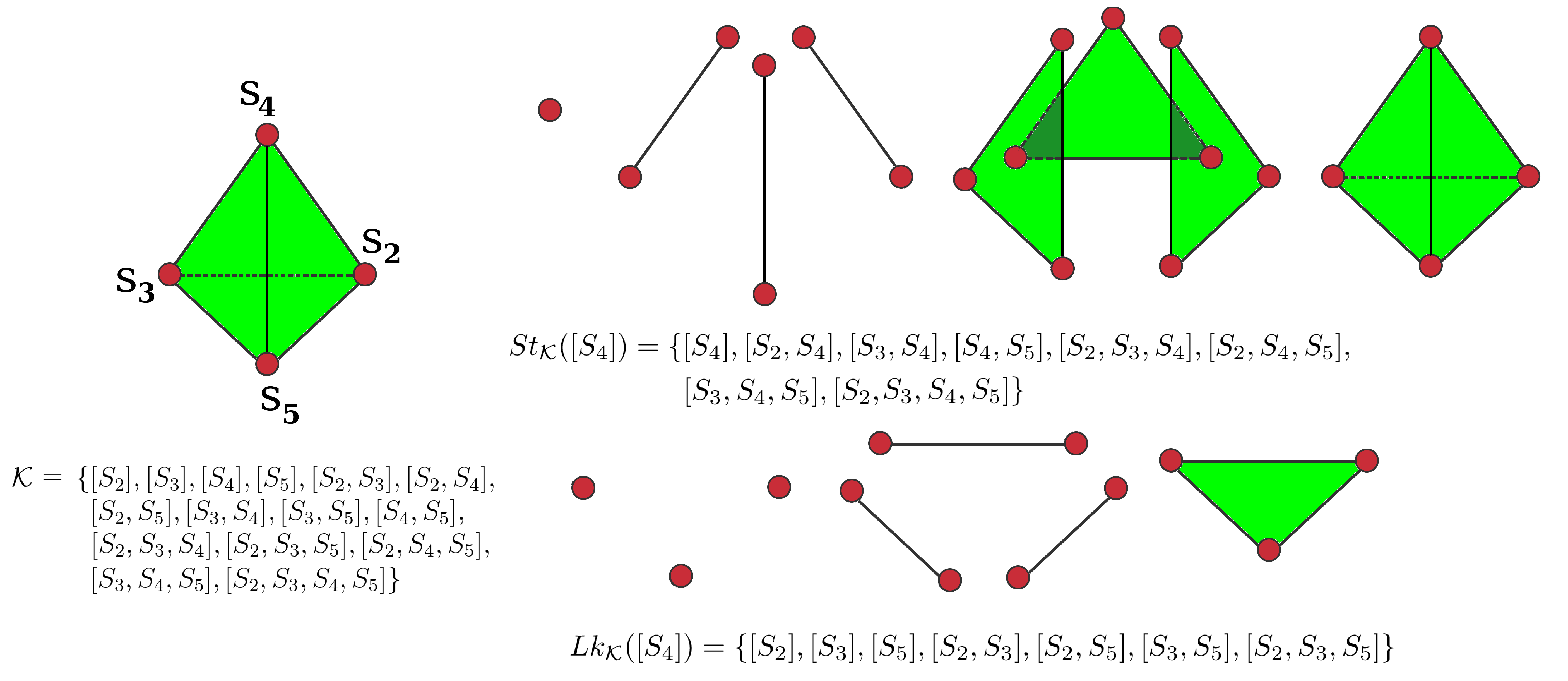}
 \caption{Example of $St_\mathcal{K}([s_4])$ and $Lk_\mathcal{K}([s_4])$ on a given simplicial complex (a tetrahedron and all its faces) $\mathcal{K}$.}
 \label{fig:linkstar}
\end{figure}

\subsection{Persistent Homology}

As a general rule, the objective of PH is to track how topological features on a topological space appear and disappear when a scale value (usually a radius) varies incrementally, in a process known as filtration~\cite{EdelsbrunnerHarer2010}.

\begin{definition}[Filtration]\label{def:filtration} Let $\mathcal{K}$ be a simplicial complex.  A filtration $\mathcal{F}$ on $\mathcal{K}$ is a succession of increasing sub-complexes of $\mathcal{K}$:
$\emptyset \subseteq \mathcal{K}_0 \subseteq \mathcal{K}_1 \subseteq \mathcal{K}_2 \subseteq \mathcal{K}_3 \subseteq \cdots \subseteq \mathcal{K}_n = \mathcal{K}$. In this case, $\mathcal{K}$ is called a filtered simplicial complex.
 \end{definition}
 
 In many simplicial complexes the simplices are determined by proximity under a distance function. A filtration $\mathcal{F}$ on a simplicial complex $\mathcal{K}$ is obtained by taking a collection $\mathcal{E}_\mathcal{K}$ of positive values $0 < \epsilon_{0}< \epsilon_{1}<\dots<\epsilon_{n}$, and the complex $\mathcal{K}_{i}$ corresponds to the value $\epsilon_{i}$. The set $\mathcal{E}_\mathcal{K}$ is called the \textit{filtration value collection} associated to $\mathcal{F}$. Every $\mathcal{K}_i \subseteq \mathcal{K}$ could be recovered with the association function $\psi_\mathcal{F}(\epsilon_i)$.

 A filtration could be understood as a method to build the whole simplicial complex $\mathcal{K}$ from a ``family'' of sub-complexes incrementally sorted according to some criteria, where each level $i$ corresponds to the ``birth'' or ``death'' of a topological feature (connected components, holes, voids). A topological feature of dimension j is a chain of j-simplices, which is not trivial in the $j$-th homology group, also referred to as a non-trivial $j$-cycle. Thus, a persistence interval (birth, death) is the ``lifetime'' of a given topological feature~\cite{EdelsbrunnerHarer2010} (see Figure~\ref{fig:filtracion}).

  \begin{figure}[H] 
    \centering
    \includegraphics[width=0.9\columnwidth]{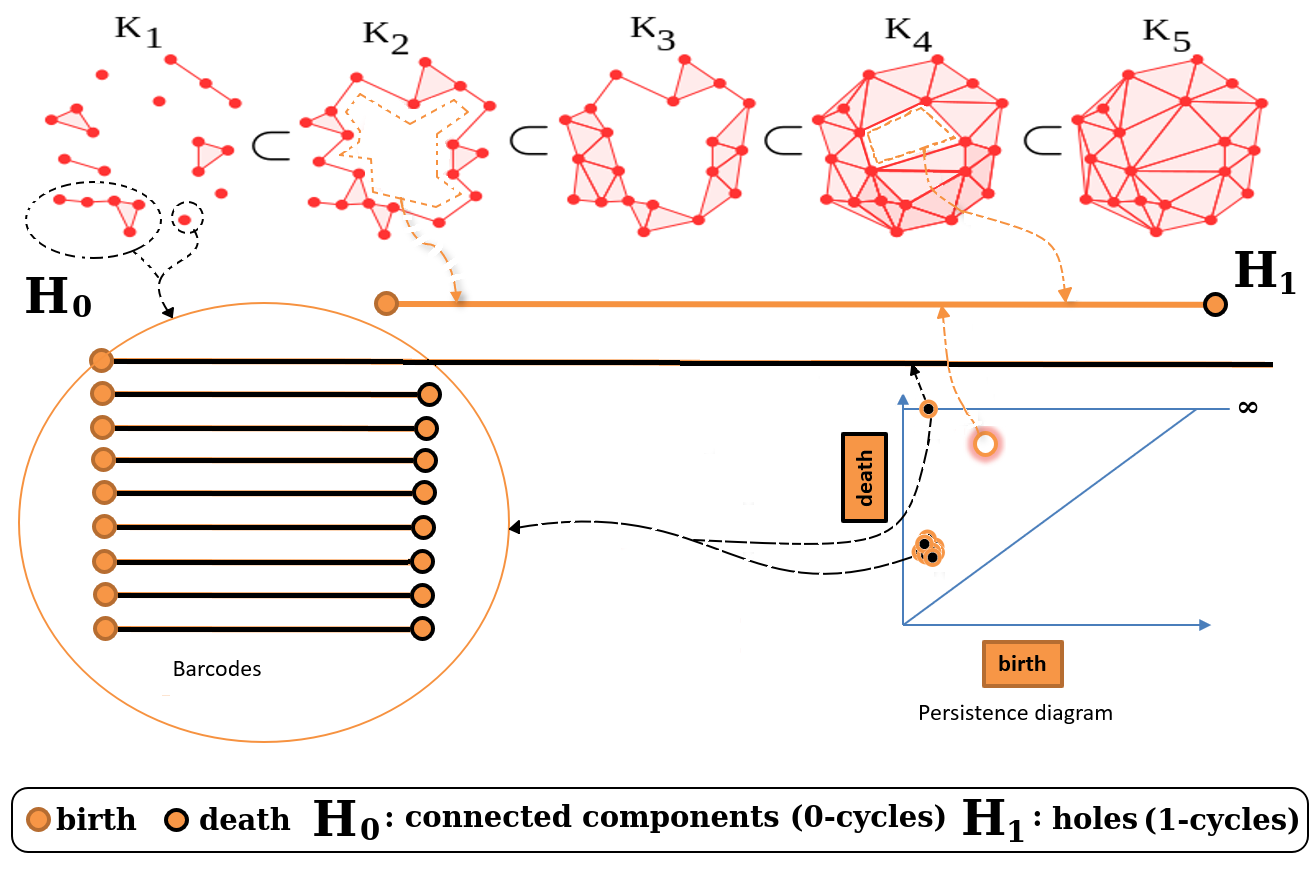}
    \caption{A fragment of a simplicial complex filtration is presented with some selected topological features. Initially, there are nine connected components, which get merged into one. Conforming a filtration value is increased topological features of higher dimensions are born and die at the end. The complete topological information is then summarized in barcodes and persistence diagrams.}\label{fig:filtracion}
\end{figure}

 \begin{definition}[Filtration value of a q-simplex]\label{def:getfvalue}  Let $\mathcal{K}$ be a filtered simplicial complex and $\mathcal{E}_\mathcal{K}$ its filtration value collection. Let $\sigma\in \mathcal{K}$ be a q-simplex. 
 If $\sigma \in \mathcal{K}_j$ but $\sigma \not \in \mathcal{K}_{j-1}$, then $ \xi_\mathcal{K}(\sigma)=\varepsilon_{j}$ is the filtration value of $\sigma$. \par
 \end{definition}
 Note that $\tau \leq \sigma \implies \xi_\mathcal{K}(\tau) \leq \xi_\mathcal{K}(\sigma)$, which means that in a filtered simplicial complex $\mathcal{K}$, every simplex $\tau \in \mathcal{K}$ appears before all its co-faces.
 
 \begin{figure}[H] 
    \centering
    \includegraphics[width=0.5\columnwidth]{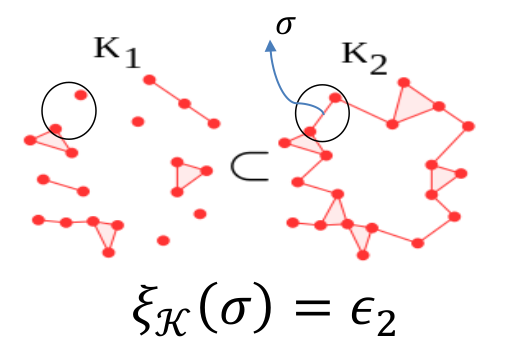}
    \caption{Filtration value of a simplex $\sigma$, which appears in $\mathcal{K}_2$ sub-complex under the $\epsilon_2$ filtration value.}\label{fig:fvalue}
\end{figure}
 
 \subsection{Classification problem}\label{scc:preliminar}

Let $\mathbb{R}^{n}$ be a feature space and $P \subset \mathbb{R}^{n}$ a finite subspace. Suppose $P$ is divided in two subspaces $P = S \cup X$, where $S$ is the training set and $X$ is the test set. Let $L=\{l_{1},\dots,l_{N}\}$ be the label set, and $\mathbb{T} = \{(p, l): p \in P, l \in L\}$ be the association space, where $\mathbb{T} = T_S \cup T_X$, $T_S$ and $T_X$ the two disjoint association sets corresponding to $S$ and $X$, respectively. The label list $Y = \{l_i \mid (x_i, l_i) \in T_X\}$, is the list of labels assigned to each element of $X$ in the association set $T_X$. 
Thus, the classification problem could be defined as how to predict a suitable label $l \in L$ for every $x \in X$ by assuming the association set $T_X$ is unknown. Consequently, the predicted label list, $\hat{Y} \subset L^{\abs{T_X}}$, will be the collection of labels resulting from the classification method. 

\section{Proposed Classification Method}\label{scc:classifier}

A classification method based on TDA is presented in this section. Overall, a filtered simplicial complex $\mathcal{K}$ is built over $P$ to generate data relationships.
~The proposed method is based on the assumption that on the filtration, a sub-complex $\mathcal{K}_i\subset \mathcal{K}$ exists, whose simplices represent a feasible approximation to the data topology.
~The fact that a point set $\{v_0, v_1,\dots, v_q\} \subset P$ defines a q-simplex $\sigma \in \mathcal{K}$ implies a similarity or dissimilarity relationship between the points $v_0, v_1,\ldots, v_q$. This implicit relationship among data is applied by the proposed method to propagate labels from labeled points to unlabeled points. In Figure~\ref{fig:overall}, the proposed method is illustrated by applying a 4-step process to classify two unlabeled points $x_1,x_2 \in X$. \par  

\begin{figure}[H]
    \centering
    \includegraphics[width=1\columnwidth]{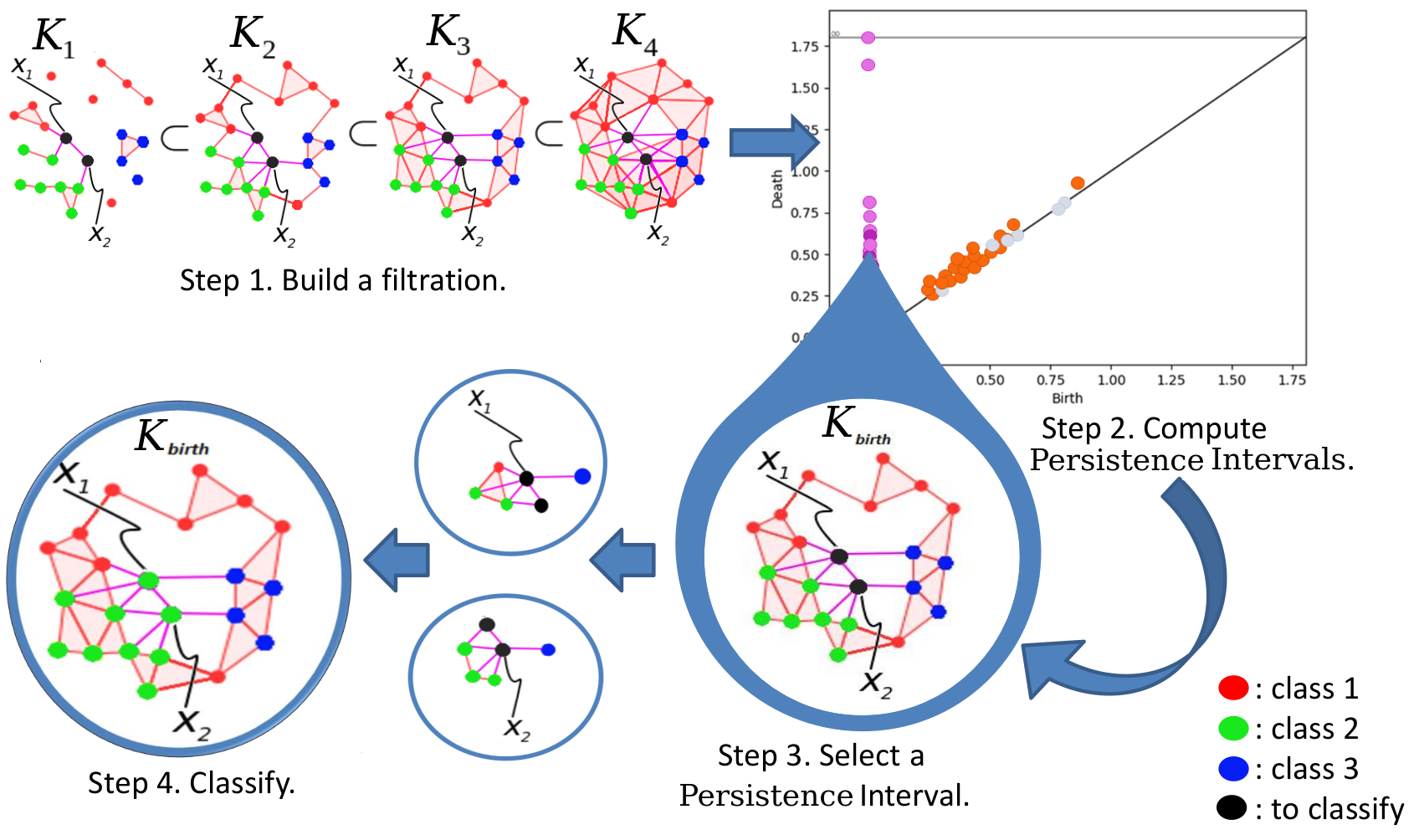}
    \caption{Overall TDABC algorithm.} \label{fig:overall}
\end{figure}

\subsection{Step 1. Building the filtered simplicial complex.}\label{scc:construction}

A filtered simplicial complex $\mathcal{K}$ is built on the dataset $P = S \cup X$ using a distance or proximity function which is problem-specific (Euclidean, Manhattan, Cosine, and more). A maximal dimension $2 \leq q \ll \abs{P}$ is given to control the simplicial complex exponential growing. 
We apply the edge collapsing method~\cite{boissonnat2020} to reduce the number of simplices but maintaining the same persistence information of the original simplicial complex. The implementation details are available as supplementary material.

\subsection{Step 2 \& 3. Recover a meaningful sub-complex}\label{scc:step2}

A filtered simplicial complex $\mathcal{K}$ provides a large quantity of multiscale data relationships. Thereby, we would like to choose a sub-complex $\mathcal{K}_i$ from the filtered simplicial complex $\mathcal{K}$, to approximate the actual structure of the dataset. 
In this vein, we exploit the ability of PH to detect topological features. For a filtered simplicial complex $\mathcal{K}$ of dimension $q$, PH will compute up to $q$-dimensional homology groups.
Each topological feature represented by an element in a  homology group of a given dimension will be represented by a persistence interval $(birth, death)\subset\mathbb{R}$. Let $D^{j}$ be the set of the persistent intervals of non-trivial $j$-cycles along the filtration of $\mathcal{K}$. We then collect all persistence intervals $D = \bigcup_{j>0} D^j$. The 0-dimensional homology group is excluded because we aim to minimize the connected components while looking at homology groups in higher dimensions.\par 
We define the $int(\cdot)$ function (Equation~\ref{eq:intd}) which measures the lifetime of the topological feature associated with each $d \in D$. The immortal persistence intervals (infinite death values) are truncated to the maximum value in the collection of filtration values $max(\mathcal{E}_\mathcal{K})$.

\begin{equation}\label{eq:intd}
    int(d):=min\{d[death],max(\mathcal{E}_\mathcal{K})\}-d[birth], 
\end{equation}

A persistence interval $d \in D$ is selected by using the functions defined on Equations~\ref{eq:maxint},\ref{eq:randint},\ref{eq:avgint}:  

\begin{enumerate}[(a)]
    \item The persistence interval with maximal persistence: 
    \begin{equation}\label{eq:maxint}
      d_m = MaxInt(D) = \argmax_{d \in D}(int(d)).
    \end{equation}    
    \item A persistence interval selected randomly on the upper half persistence intervals of highest-lifespan: 
    \begin{equation}\label{eq:randint}
      d_r = RandInt(D) = random(\{d \mid int(d) > avg(D)\}).
    \end{equation}
    \item The closest interval to the persistence intervals average: 
    \begin{equation}\label{eq:avgint}
      d_a = AvgInt(D) = \argmin_{d \in D} (\abs{int(d) - avg(D) }),
    \end{equation}
    where $avg(D) = \frac{1}{\abs{D}} \cdot \sum_{d_i \in D} int(d_i).$
\end{enumerate}
 If a tie occurs during the computation of $MaxInt(D)$ or $AvgInt(D)$, the persistence interval with higher birth time should be taken. Once $d \in \{d_m, d_r, d_a\}$ is chosen, a sub-complex $\mathcal{K}_i$ must be selected. We choose an appropriate filtration value from the birth, death or middle-time of persistence interval d: 
$$\varepsilon_i \in \left\{d[birth], d[death], \frac{d[birth]+d[death]}{2}\right\},$$ 
then the respective sub-complex $\mathcal{K}_i = \psi_\mathcal{F}(\varepsilon_i)$ is obtained.\par 
 
We collect all simplices born on the lifespan of the selected persistence interval and compute their closure to get the minimal simplicial complex that contains them.\par
In Figure~\ref{fig:circlepinfo}, 
$\mathcal{K}_i \subseteq \mathcal{K}$ is chosen by using PH to guide the selection of candidate persistence intervals. Then, a sub-complex is recovered on the death of the persistence interval selected according to the $MaxInt(\cdot)$,  $RandInt(\cdot)$ selection functions (see Equation~\ref{eq:maxint} and Equation~\ref{eq:randint}). In this example, the filtered simplicial complex $\mathcal{K}$ was built on the Circles dataset (noise = $10$), see details in Section~\ref{scc:results}. 

\begin{figure}[t]
\begin{center}
\includegraphics[width=0.8\columnwidth]{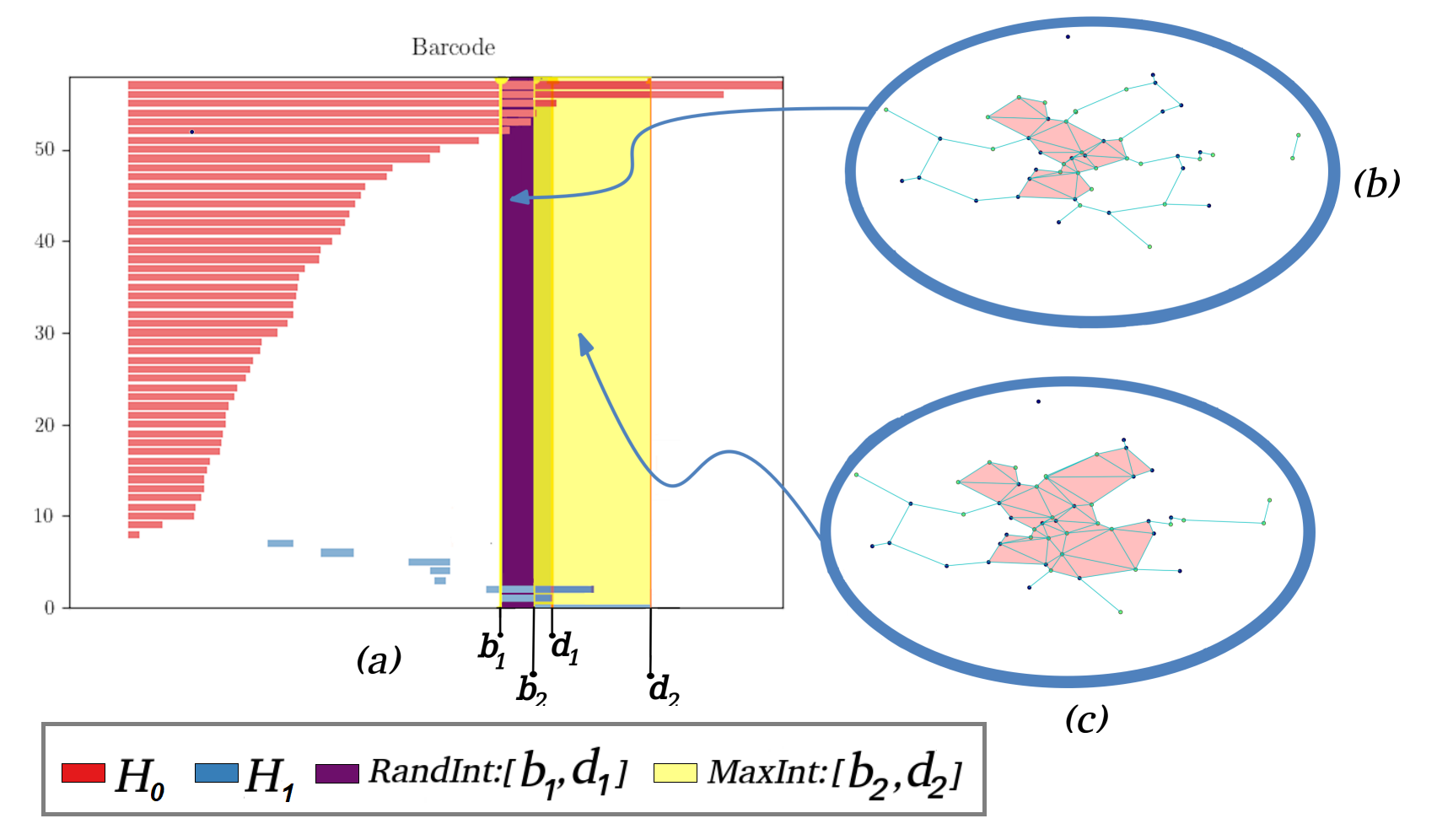}
\caption{A barcode representation of PH in $\mathcal{K}$ is shown in (a), which have two homology groups $H_0$ (group of connected components) and $H_1$ (group of 1-cycles). We show two persistence intervals $[b_1, d_1]$ (purple), and $[b_2, d_2]$ (yellow) corresponding to the $RandInt(\cdot)$ and $MaxInt(\cdot)$ selection functions, respectively. In (b) and (c), respectively, the sub-complexes $\mathcal{K}_{d_1} \subseteq \mathcal{K}$ and $\mathcal{K}_{d_2} \subseteq \mathcal{K}$ are shown.}\label{fig:circlepinfo}    
    
\end{center}

\end{figure}

\subsection{Step 4. Classify} \label{scc:labeling}

\begin{figure*}[htb]
    \centering
    \includegraphics[width=1\textwidth]{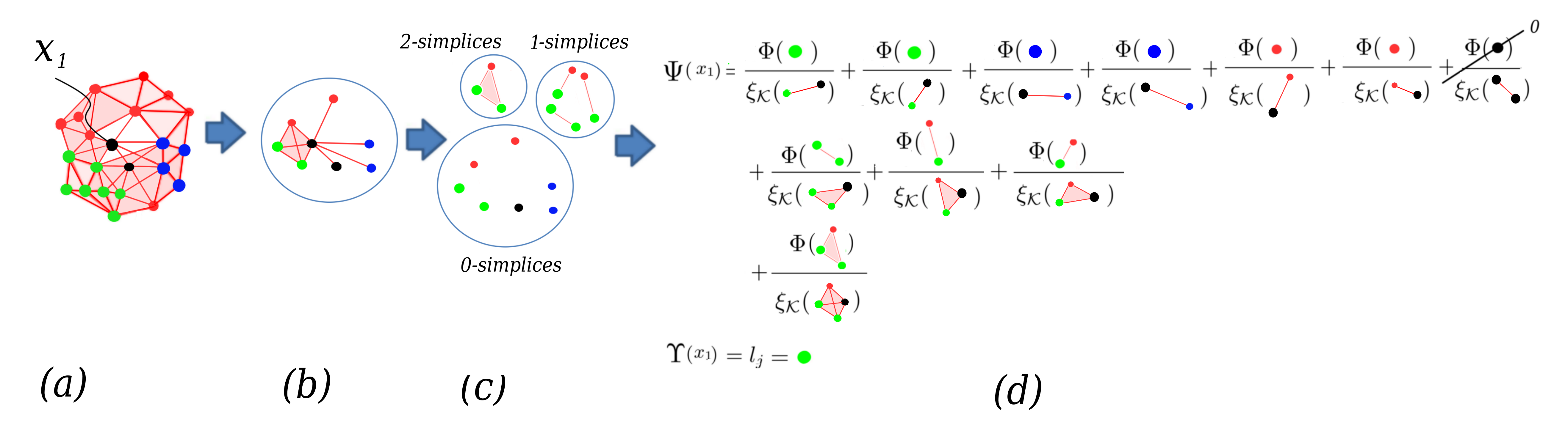}
    \caption{From a selected sub-complex $\mathcal{K}_i \subseteq \mathcal{K}$ in \textbf{(a)}, the star of $x_1$ is obtained in \textbf{(b)}. In  \textbf{(c)}, $Lk([x_1])$ is shown divided into 0-simplices, 1-simplices, and 2-simplices. In \textbf{(d)}, the extension function $\Psi_{i}(x_1)$ is executed, and finally the labeling function $\Upsilon_{i}(x_1)$ assigns a green label to $x_1$. The $\xi_\mathcal{K}$ function acts as a tie-breaker of green/red/blue label contributions, even when, in section \text{(b)}, might seem a tie between labels.}\label{fig:epsilon}
\end{figure*}

The neighborhood relationships of a q-simplex $\sigma \in \mathcal{K}$ will be recovered by using the link, star, and closed star (Definition~\ref{def:starlink}). 
A key component of the proposed method is the label propagation over a filtered simplicial complex detailed in this section.

Suppose a preferred sub-complex $\mathcal{K}_{i} \in \mathcal{K}$ has been selected according to section~\ref{scc:step2}. Let $A$ be the $\mathbb{R}$-module with generators $\hat{l}_1, \hat{l}_2,\dots, \hat{l}_N$ with $N = \abs{L}$. We consider $0 \in A$ to represent no-value. The generator $\hat{l}_j$ will be associated to the label $l_{j}$ according to Definition~\ref{def:assoc_funct}.

\begin{definition}[Association function]\label{def:assoc_funct} Let  $\Phi_{i}: \mathcal{K}_{i} \rightarrow A$ be the association function defined on a 0-simplex $v \in \mathcal{K}_{i}$ as $\Phi_{i}(v) =\hat{l}_{j}$ if $(v,l_{j})\in T_S$ and $\Phi_{i}(v)=0$ in any other case. The association function can be extended to a q-simplex $\sigma$ as $\Phi_{i}(\sigma)= \sum_{v\in \mathcal{V}(\sigma)}\Phi_{i}(v)$.
\end{definition}

In a simplicial complex, we can use the link operation to propagate labels from labeled points in $S$ to unlabeled points in $X$. To do so, we define the extension function as follows:

\begin{definition}[Extension function]\label{def:assoc_func_psi} Let $\Psi_{i}: X \rightarrow A$ be the function defined on a point $v\in X$ by
\begin{equation}\label{eq:extension}
 \Psi_{i}(v)=\sum_{\sigma\in Lk_{\mathcal{K}_{i}}([v])}\frac{\Phi_{i}(\sigma)}{\xi_\mathcal{K}([\mathcal{V}(\sigma) \cup \{v\}]}
\end{equation}

According to Lemma~\ref{lemma:002}, we can also obtain an equivalent formula

    \begin{equation} \label{eq:extension_simplex}
      \Psi_{i}(v)=\sum_{\mu\in St_{\mathcal{K}_{i}}([v])}\frac{\Phi_{i}([\mathcal{V}(\mu)\setminus \{v\}])}{\xi_\mathcal{K}(\mu)}
    \end{equation}
\end{definition}    
     
 In Equation~\ref{eq:extension} and Equation~\ref{eq:extension_simplex}, we obtain the co-faces of $v$ such that $\mu \in St_{\mathcal{K}_i}([v]),~\mu = [\mathcal{V}(\sigma) \cup \{v\}],~\sigma \in Lk_{\mathcal{K}_i}([v])$ according to Lemma~\ref{lemma:002}. The filtration value $\xi_\mathcal{K}(\mu)$ is applied to prioritize the influence of $\sigma$ to label $v$. Let $\alpha, \beta \in St_{\mathcal{K}_{i}}([v])$ be two simplices, such that $\xi_\mathcal{K}(\alpha) < \xi_\mathcal{K}(\beta)$. This condition implies that $\alpha$ was clustered around $v$ earlier than $\beta$ was since $\alpha$ appears before $\beta$ in the filtration. In consequence, the $\alpha$ contributions should be more important than the $\beta$ contributions. Observe that using filtration values as inverse weights makes the classification aware of the filtration history and provides several properties such as distance encoding operators, indirect local outlier factors (because it excludes too far points), and density estimators (provides short values for dense regions or small q-simplices). Figure~\ref{fig:epsilon} section (d), shows the impact of $\xi_\mathcal{K}(\cdot)^{-1}$ to classify an unlabeled point $x_1\in X$.\par

According to the previous definitions, given a point $v\in X$, the evaluation of the extension function at $v$ would be $\Psi_{i} (v)=\sum_{j=1}^{N} a_{j}\cdot \hat{l}_{j}$, where $a_{j}\in \mathbb{R}^{+}\cup\{0\}, j=1,\dots, N$. \par

\begin{definition}[Labeling function]\label{def:lab_func}  Let $v$ be a point in $X$ such that $\Psi_{i} (v)=\sum_{j=1}^{N} a_{j}\cdot \hat{l}_{j}$. Let $\tilde{a}$ be the maximum value in $\{a_{j}\}_{j=1}^{N}$, and $\tilde{A} = \{j\mid a_{j}=\tilde{a}\}$ be the set of maximum value indexes. We define the labeling function $\Upsilon_i$ at $v$ as  $\Upsilon_i(v) =  l_k$ where $k$ is uniformly selected at random from $\tilde{A}$. If $\tilde{a} = 0$ then $\Upsilon_i(v) = \emptyset$.
\end{definition} 

If there is a unique maximum in the set $\{a_j\}_{j=1}^{N}$ from the previous definition, the labeling function is uniquely defined at $v$. In all tested datasets, the label assignment of each point in $X$ was uniquely defined because the factor $\frac{1}{\xi_\mathcal{K}(\cdot)}$ acts as a tie-breaker. Figure~\ref{fig:epsilon} shows the labeling process on a previously selected sub-complex, and the classification process is summarized in Algorithm~\ref{alg:labeling}. 

\begin{algorithm}[H]
\caption{\textbf{Labeling:} Labeling a point set $X$.}\label{alg:labeling}
\begin{algorithmic}[1]
\REQUIRE A filtered simplicial complex $\mathcal{K}$. A non-empty point set $X$.
\ENSURE A predicted labels list $\hat{Y}$ of $X$.

\STATE $D \leftarrow GetPersistenceIntervalSet(\mathcal{K})$ where:\newline
$D = \{d_i \mid d_i=(birth, death)\}$ 
\STATE Get a desired persistence interval $d$ where: \newline $d \in \{MaxInt(D), RandInt(D), AvgInt(D)\}$\label{step: selection} 
\STATE $\varepsilon_i \leftarrow d[birth];~\mathcal{K}_{i} \leftarrow \psi_\mathcal{F}(\varepsilon_i);~\hat{Y} \leftarrow \{\}$
\WHILE{$X \neq \emptyset$}
\STATE $v \in X;~l \leftarrow \Upsilon_{{i}}(v)$ 
\IF{$l = \emptyset$} 
\STATE $l \leftarrow handling\_special\_cases(v)$ 
\ENDIF
\STATE $\hat{Y} \leftarrow \hat{Y} \cup \{l\};~X \leftarrow X \setminus \{v\}$
\ENDWHILE
\RETURN $\hat{Y}$
\end{algorithmic}
\end{algorithm}

\subsection{Dealing with special cases}\label{scc:border}

There are two cases where $\tilde{a} = 0$ in Definition~\ref{def:lab_func}. When a point $v$ is isolated (i.e., $v \notin P$, thus $Lk_{\mathcal{K}_i}([v]) = \emptyset$), or when all points in $Lk_{\mathcal{K}_i}([v])$ are unlabeled.

\subsubsection{Isolated points}  
To handle isolated points, we look for a collection of points close to $v$. This collection is defined by  $U_v = \{u \in \mathcal{V}(\mathcal{K}_0) \mid f(u, v) \leq 2\cdot \varepsilon_i\}$, where $f(\cdot)$ is the distance or dissimilarity function applied to build $\mathcal{K}$. Where $\varepsilon_i = d[death]$ is the death time of the chosen persistence interval $d$ to recover $\mathcal{K}_i$~(Section~\ref{scc:step2}). Since $U_v$ could be a non-disjoint collection of 0-simplices on $\mathcal{K}_i$, we compute their label contributions regarding $v$, with $\Psi (v) = \sum_{u \in U_v} \frac{\Psi_i (u)}{f(u, v)}$. We use  $\frac{1}{f(\cdot, \cdot)}$ to give more importance to label contributions according to the closeness to $v$. Then, $\Upsilon_i(v)$ is performed according to Definition~\ref{def:lab_func}. Note that when $U_v \subseteq X$, we reach the second case.\par

\subsubsection{Unlabeled link}

To address the case when all points in $Lk_{\mathcal{K}_i}([v])$ are unlabeled, we look for the shortest paths from $v$ to the labeled points. This problem can be considered ``semi-supervised learning" by constraining the domain to a neighborhood around the link. We propose a heuristic to reach a fast convergence by using a priority queue $Q$. We insert all $\sigma \in St_{\mathcal{K}_i}([v])$ in $Q$ by using their filtration values as priority $\rho(\sigma) = \xi_\mathcal{K}(\sigma)$. We also maintain a flag to avoid processing simplices more than once. While $Q$ is not empty, we process $\tau \in Q$ with priority $\rho(\tau)$ by quantifying its label contributions with $\sum_{\tau \in Q}\sum_{\mu \in St_{\mathcal{K}_i}(\tau)}\frac{\Phi_i(\mu)}{\rho(\tau) + \xi_\mathcal{K}(\mu)} = \{a_j\}^N_{j=1}$. Every non-visited $\mu$ such that $\mathcal{V}(\mu) \subseteq X$, is added to $Q$ with priority $\rho(\mu) = \rho(\tau) + \xi_\mathcal{K}(\mu)$. At the end, we report the label with majority of votes according to Definition~\ref{def:lab_func}. 
  
\section{Experimental Results}\label{scc:results}

 Our TDA based classifier (TDABC) was evaluated considering the selection functions: $RandInt(\cdot)$ or TDABC-R, $MaxInt(\cdot)$ or TDABC-M, and $AvgInt(\cdot)$ or TDABC-A. Four baseline methods were selected to compare the proposed methods: k-Nearest Neighbors (KNN), distance-based weighted k-NN (WKNN), Linear Support Vector Machine (LSVM), and Random Forest (RF). All baseline classifiers were manually configured to deal with imbalanced datasets using the known class frequencies (``class\_weight'' parameter in Scikit-Learn Library~\cite{scikitlearn}). Table~\ref{tb:datasets} shows the datasets and their characteristics.

\begin{table}[htb]
\centering
\caption{Selected datasets to evaluate proposed and baseline classifiers.}
\label{tb:datasets}
\resizebox{\columnwidth}{!}{%
\begin{tabular}{|l|l|l|l|l|l|l|l|}
\hline
\textbf{Name} & \textbf{Dimensions} & \textbf{Classes} & \textbf{Size} & \textbf{\begin{tabular}[c]{@{}l@{}}Samples\\ per class\end{tabular}} & \textbf{Noise} & \textbf{Mean} & \textbf{Stdev} \\ \hline
\textbf{Circles} & 2 & 2 & 50 & {[}25,25{]} & 3 & - & - \\ \hline
\textbf{Moon} & 2 & 2 & 200 & {[}100,100{]} & 10 & - & - \\ \hline
\textbf{Swissroll} & 3 & 6 & 300 & {[}50,50,50,50,50,50,{]} & 10 & - & - \\ \hline
\textbf{Iris} & 4 & 3 & 150 & {[}50,50,50{]} & - & - & - \\ \hline\hline
\textbf{Normal} & 350 & 5 & 300 & {[}60,10,50,100,80{]} & - & {[}0,0.3,0.18,0.67,0{]} & 0.3 \\ \hline
\textbf{Sphere} & 3 & 5 & 653 & {[}500,100,25,16,12{]} & - & 0.3 & 0.147 \\ \hline
\textbf{Wine} & 13 & 3 & 178 & {[}59, 71, 48{]} & - & - & - \\ \hline
\textbf{Cancer} & 30 & 2 & 569 & {[}212, 357{]} & - & - & - \\ \hline
\end{tabular}%
}
\end{table}

\subsection{Artificial datasets}\label{scc:artificial}

The \textbf{\textit{Circles, Swissroll, Moon, Norm,}} and \textbf{\textit{Sphere}} datasets were artificially generated. In the case of \textbf{\textit{Circles, Moon,}} and \textbf{\textit{Swissroll}} a Gaussian noise factor was added to diffuse per-class boundaries and to assess classification performance for overlapped data regions.  \par

The \textbf{\textit{Normal}}, and \textbf{\textit{Sphere}} datasets were generated based on a Normal distribution per dimension. The Normal dataset has a high dimension ($P \subset \mathbb{R}^{350}$, $\abs{P} < 350$). The Sphere dataset is always in three dimensions ($P \subset \mathbb{R}^3$), aiming to capture entanglement and imbalance sample distributions.

\subsection{Real-world datasets}
The Iris, Wine, and Cancer datasets were selected as real datasets to compare the proposed classifiers and the baseline ones. The \textbf{\textit{Iris dataset}}~\cite{Duadua:2019} is a balanced dataset where one class is linearly separable and the other two are slightly entangled each other. The \textbf{\textit{Wine dataset}}~\cite{Duadua:2019} is an imbalanced dataset with thirteen different measurements to classify three types of wine. 
 The \textbf{\textit{Breast Cancer dataset}} (Cancer)~\cite{Duadua:2019} is an imbalanced dataset with thirty features and two classes. The Wine and Cancer datasets were transformed using a logarithmic statistical transformation. Accordingly, a resulting dataset was obtained $P' = \{\ln{(p + M)}\}_{p \in P}$, with $M$ the minimum component value of the dataset employed to deal with negative numbers. On the other datasets, no transformation was required. 
 
\subsection{Evaluation methodology}

 We build baseline and proposed classifiers using the Euclidean distance in all datasets. \par 

The classifier evaluation across datasets was conducted using a Repeated R-Fold Cross-Validation process ($10$\% fold, N=5). We divide results according to the balancing condition of the datasets, imbalanced dataset results in Table~\ref{tb:tabla_imbalanced} and balanced dataset results are shown in Table~\ref{tb:table_balanced}. We compute the following metrics: $\text{F1}=2 \cdot \frac{\text{Precision} \cdot \text{Recall}}{\text{Precision} + \text{Recall}}$, the harmonic mean between precision and recall, we are considering both with the same importance. The ROC-AUC curve with one-vs-rest approach and macro average. The Precision-Recall (PR) curve and report the Average Precision metric or (PR-AUC). The True Negative Ratio or Specificity ($TNR = \frac{TN}{TN+FP}$), False Positive Ratio ($FPR = \frac{FP}{TP+FP}$), and the Geometric Mean ($GMEAN = \sqrt{TNR \cdot Recall}$) between precision and recall. In each metric a high value is better than a lower one, except in FPR where a lower value is preferred. Figure~\ref{fig:all_metric} shows the results, on the imbalance datasets we report the metric results on the minority class. 

 \begin{figure}[H]
\centering
\includegraphics[width=0.8\columnwidth]{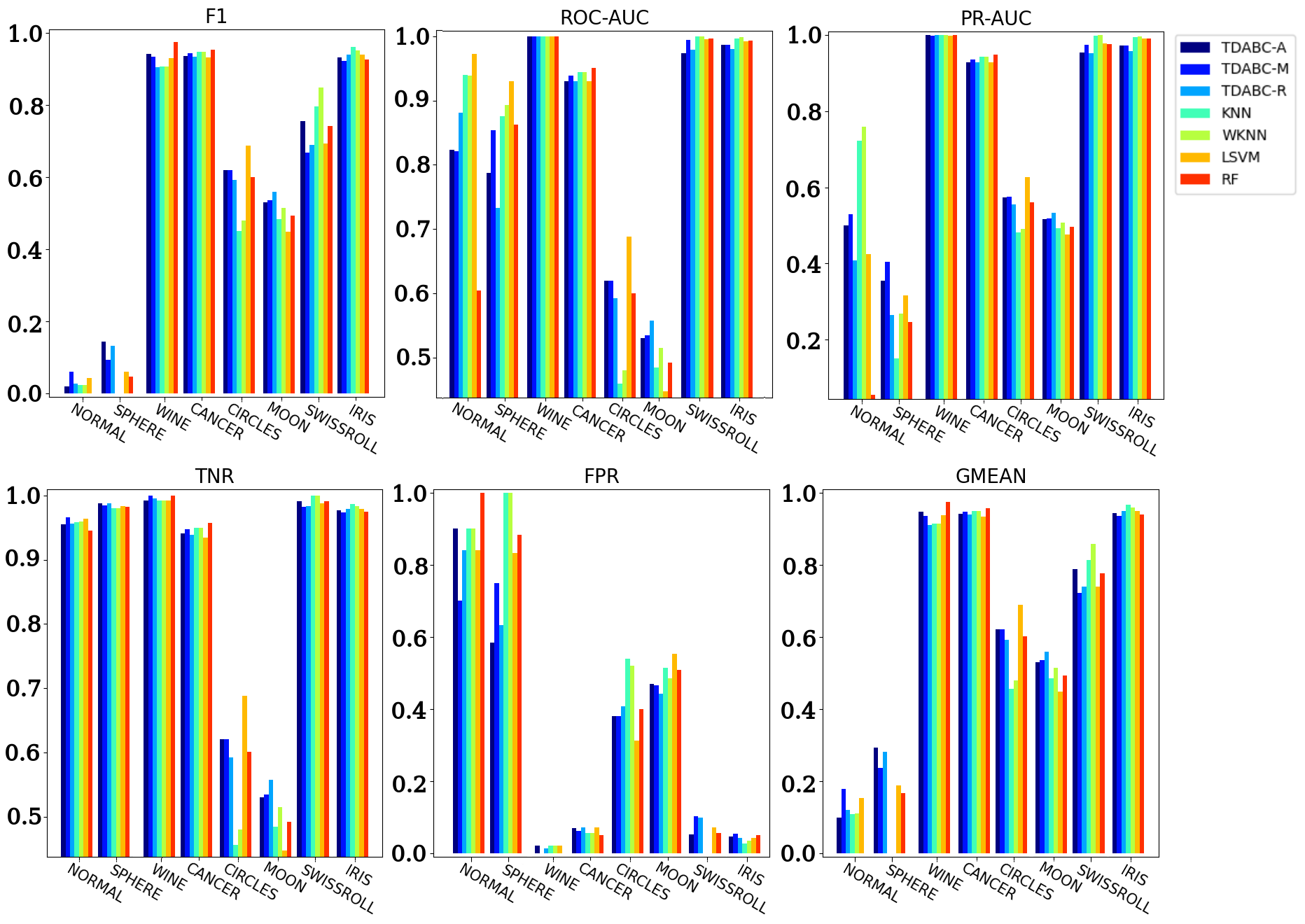}
\caption{We present results of the F1, ROC-AUC, PR-ROC, TNR, FPR, and GMEAN metrics across datasets, in case of imbalanced datasets the results of minority classes were shown, since it is commonly considered the more important class.}\label{fig:all_metric}
\end{figure}


\begin{table}[!htb]
\centering
\caption{Metric results per classifier across imbalanced datasets. We show global metric results and the results on the minority (Min) class. In black those classifiers that were superior to the arithmetic mean.}
\label{tb:tabla_imbalanced}
\resizebox{0.8\columnwidth}{!}{%
\begin{tabular}{l|ll|ll|ll|ll}\hline
\multicolumn{1}{c}{\multirow{2}{*}{\textbf{CLASSIFIERS}}} & \multicolumn{2}{|l}{\textbf{NORMAL}} & \multicolumn{2}{|l}{\textbf{SPHERE}} & \multicolumn{2}{|l}{\textbf{WINE}} & \multicolumn{2}{|l}{\textbf{CANCER}} \\
\multicolumn{1}{c}{} & \multicolumn{1}{|c}{\textbf{Global}} & \multicolumn{1}{c}{\textbf{\begin{tabular}[c]{@{}c@{}}Min \\ class\end{tabular}}} & \multicolumn{1}{|c}{\textbf{Global}} & \multicolumn{1}{c}{\textbf{\begin{tabular}[c]{@{}c@{}}Min \\ class\end{tabular}}} & \multicolumn{1}{|c}{\textbf{Global}} & \multicolumn{1}{c}{\textbf{\begin{tabular}[c]{@{}c@{}}Min \\ class\end{tabular}}} & \multicolumn{1}{|c}{\textbf{Global}} & \multicolumn{1}{c}{\textbf{\begin{tabular}[c]{@{}c@{}}Min \\ class\end{tabular}}} \\ \hline
\multicolumn{9}{c}{\textbf{F1}} \\ \hline
\textbf{TDABC-A}  & 0.351          & 0.019          & \textbf{0.436} & \textbf{0.143*} & \textbf{0.950} & \textbf{0.941} & 0.935           & 0.917 \\
\textbf{TDABC-M}  & \textbf{0.397} & \textbf{0.059*} & \textbf{0.412} & \textbf{0.092} & \textbf{0.942} & \textbf{0.933} & \textbf{0.943}  & \textbf{0.928} \\
\textbf{TDABC-R}  & 0.336          & \textbf{0.027} & \textbf{0.441} & \textbf{0.132} & 0.917           & 0.905         & 0.934           & 0.916 \\
\textbf{KNN}      & \textbf{0.411} & 0.022          & 0.342          & 0.000          & 0.918           & 0.906         & \textbf{0.947}  & \textbf{0.933} \\
\textbf{WKNN}     & \textbf{0.412} & 0.022          & 0.397          & 0.000          & 0.918           & 0.906          & \textbf{0.947} & \textbf{0.933} \\
\textbf{LSVM}     & \textbf{0.425} & \textbf{0.042} & 0.389          & 0.059          & \textbf{0.940}  & \textbf{0.930} & 0.932          & 0.914 \\
\textbf{RF}       & 0.250          & 0.000          & \textbf{0.427} & 0.046 & \textbf{0.978}  & \textbf{0.974*} & \textbf{0.954} & \textbf{0.941*} \\ \hline
\textbf{Average}  & 0.369          & 0.027          & 0.407          & 0.067          & 0.938           & 0.928          & 0.942          & 0.926 \\ \hline

\multicolumn{9}{c}{\textbf{PR-AUC}} \\ \hline
\textbf{TDABC-A} & 0.765          & \textbf{0.500}  & \textbf{0.904}    & \textbf{0.355}    & 0.978         & \textbf{1.000*}  & 0.928          & 0.928 \\
\textbf{TDABC-M} & 0.739          & \textbf{0.529}  & 0.961             & \textbf{0.405*}    & 0.977         & 0.999           & \textbf{0.937} & \textbf{0.937} \\
\textbf{TDABC-R} & 0.706          & 0.408           & \textbf{0.914}    & 0.265             & 0.977         & \textbf{1.000*} & 0.928          & 0.928 \\
\textbf{KNN}     & \textbf{0.796} & \textbf{0.724}  & 0.958             & 0.151             & \textbf{0.990} & \textbf{1.000*} & \textbf{0.943} & \textbf{0.943} \\
\textbf{WKNN}    & \textbf{0.817} & \textbf{0.760*}  & \textbf{0.969}   & 0.269            & \textbf{0.993}  & \textbf{1.000*} & \textbf{0.943} & \textbf{0.943} \\
\textbf{LSVM}    & \textbf{0.872} & 0.424            & 0.964            & \textbf{0.316}    & \textbf{0.994}   & 0.999          & 0.929         & 0.929 \\ 
\textbf{RF}      & 0.730          & 0.054           & \textbf{0.973}   & 0.246             & \textbf{0.999}  & \textbf{1.000*} & \textbf{0.949} & \textbf{0.949*}        \\ \hline
\textbf{Average} & 0.775          & 0.486           & 0.949            & 0.287             & 0.987           & 1.000           & 0.937          & 0.937 \\ \hline

\multicolumn{9}{c}{\textbf{ROC-AUC}} \\ \hline
\textbf{TDABC-A} & 0.875          & 0.823            & 0.822           & 0.787              & 0.992         & \textbf{1.000*}    & 0.930             & 0.930 \\
\textbf{TDABC-M} & 0.866          & 0.821           & \textbf{0.905}   & \textbf{0.854}     & 0.991         & \textbf{1.000*}    & \textbf{0.939}   & \textbf{0.939} \\
\textbf{TDABC-R} & 0.864          & 0.881           & 0.822            & 0.733              & 0.990         & \textbf{1.000*}    & 0.930            & 0.930 \\
\textbf{KNN}     & \textbf{0.910} & \textbf{0.940}  & \textbf{0.937}   & \textbf{0.876}     & \textbf{0.998} & \textbf{1.000*}    & \textbf{0.945}    & \textbf{0.945} \\
\textbf{WKNN}    & \textbf{0.914} & \textbf{0.939}  & \textbf{0.948}   & \textbf{0.893}     & \textbf{0.999}  & \textbf{1.000*}   & \textbf{0.945}    & \textbf{0.945} \\
\textbf{LSVM}    & \textbf{0.929} & \textbf{0.973*}  & 0.941             & \textbf{0.930*}  & \textbf{0.997} & \textbf{1.000*}   & 0.930             & 0.930 \\ 
\textbf{RF}      & 0.789          & 0.605           & \textbf{0.950}   & \textbf{0.863}      & \textbf{0.999} & \textbf{1.000*}    & \textbf{0.951}    & \textbf{0.951*}        \\ \hline
\textbf{Average} & 0.878           & 0.855          & 0.904            & 0.848              & 0.995           & 1.000             & 0.938            & 0.938 \\ \hline

\end{tabular}%
}
\end{table}

\begin{table}[!htb]
\centering
\caption{Metric results per classifier across balanced datasets. In black those classifiers that
were superior to the arithmetic mean.}
\label{tb:table_balanced}
\resizebox{0.7\columnwidth}{!}{%
\begin{tabular}{l|l|l|l|l} \hline
\textbf{CLASSIFIERS} & \textbf{CIRCLES} & \textbf{MOON} & \textbf{SWISSROLL} & \textbf{IRIS} \\ \hline
\multicolumn{5}{c}{\textbf{F1}} \\ \hline

\textbf{TDABC-A}    & \textbf{0.620}    & \textbf{0.505}    & \textbf{0.829*}       & 0.932 \\
\textbf{TDABC-M}    & \textbf{0.620}    & \textbf{0.509*}   & \textbf{0.813}        & 0.922 \\
\textbf{TDABC-R}    & \textbf{0.592}    & \textbf{0.503}    & \textbf{0.805}        & \textbf{0.939} \\
\textbf{KNN}        & 0.449             & 0.445             & 0.738                 & \textbf{0.961} \\
\textbf{WKNN}       & 0.480             & 0.463             & \textbf{0.782}        & \textbf{0.951} \\
\textbf{LSVM}       & \textbf{0.688*}   & 0.426             & 0.716                 & \textbf{0.939} \\
\textbf{RF}         & \textbf{0.599}    & \textbf{0.501}    & 0.726                 & 0.926 \\ \hline
\textbf{Average}    & 0.578             & 0.478             & 0.773                 & 0.938 \\ \hline

\multicolumn{5}{c}{\textbf{ROC-AUC}} \\ \hline
\textbf{TDABC-A}    & \textbf{0.620}    & \textbf{0.505}    & \textbf{0.993}        & 0.986 \\
\textbf{TDABC-M}    & \textbf{0.620}    & \textbf{0.510*}   & 0.990                 & 0.986 \\
\textbf{TDABC-R}    & \textbf{0.632}    & \textbf{0.504}    & 0.991                 & 0.984 \\
\textbf{KNN}        & 0.460             & 0.445             & \textbf{0.992}        & \textbf{0.996} \\
\textbf{WKNN}       & 0.480             & 0.465             & \textbf{0.996*}       & \textbf{0.998*} \\
\textbf{LSVM}       & \textbf{0.640*}   & 0.428             & 0.991                 & \textbf{0.993} \\
\textbf{RF}         & \textbf{0.580}    & \textbf{0.501}    & \textbf{0.993}        & \textbf{0.994} \\ \hline
\textbf{Average}    & 0.576             & 0.479             & 0.992                 & 0.991 \\ \hline
\end{tabular}%
}
\end{table}
Additionally, we follow the experimental setting presented in \cite{LearningNeg} to assess the classifier's behavior under imbalanced data conditions. We generate 16 two-classes datasets in $\mathbb{R}^2$ by using the Normal distribution. The label 0 (positive) samples were generated using $\mu=0$, $\sigma=1.1$, and samples of label 1 (negative) were generated using $\mu=2.0$, $\sigma=2.2$. We start generating a dataset with 100 samples, 50 per class, then we maintain the same 50 samples on the positive class and increasing the negative class with 50 samples up to 800. We perform repeated cross-validation in each dataset, computing the F1, PR-AUC, ROC-AUC, TNR, FPR, and GMEAN metric averages and their standard deviation (vertical lines). The Figure~\ref{fig:auc_curve} shows the experimental setting results per classifier.  

 \begin{figure}
\centering
\includegraphics[width=1\columnwidth]{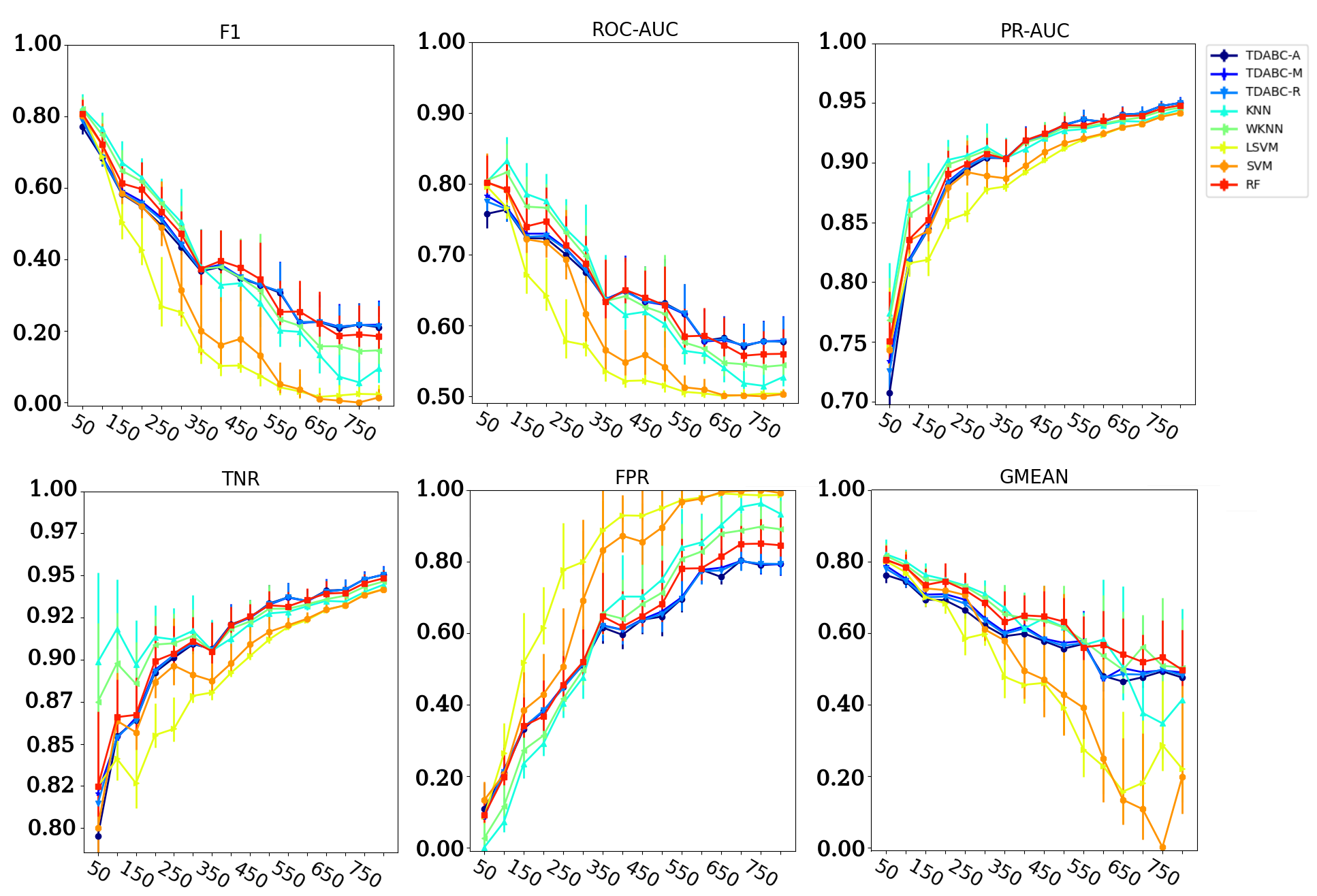}
\caption{We generate 16 datasets from 50-50 to 50-800 to evaluate how classifiers behave under several imbalance conditions. We present results of the F1, ROC-AUC, PR-ROC, TNR, FPR, GMEAN metrics in the minority class.}\label{fig:auc_curve}
\end{figure}

\subsection{Analysis}\label{sec12}

In Figure~\ref{fig:all_metric} we present results of all metrics across imbalanced and balanced datasets. 

Overall, TDABC is as good as baseline methods, but in the most imbalanced or entaggle datasets, TDABC surpasses the others. This advantage become evident in Sphere, where the imbalance ratio is up to 41:1 considering the maximal and minority classes (500:12). In a dataset with lower imbalanced ratios like Cancer and Wine, the same pattern arises with the TDABC providing similar results as baseline classifiers. Table~\ref{tb:tabla_imbalanced} shows detailed results for F1, PR-AUC, and ROC-AUC as a confirmation that TDABC behaves better classifying minority classes when a high imbalanced ratio appears. The behavior under highly imbalanced data can be seen in Figure~\ref{fig:auc_curve} where we present results under 16 different imbalanced conditions. Again, it can be seen that in F1, ROC-AUC, PR-AUC, TNR, and GMEAN metrics, the TDABC methods are conservatives from 1:1 (50:50) to 7:1 (350:50) imbalanced ratios, respectively. From 7:1 to 16:1 (800:50), TDABC becomes better than baseline methods. 

In FPR plots in Figure~\ref{fig:auc_curve}, we show that the behavior in the minority class is increased from 350:50 upwards, which means that baseline methods get wrong classifying minority classes more often than TDABC does. Interestingly, the standard deviation in our proposed methods are more stable than the baseline ones.\par


Experiments on balanced datasets, see Table~\ref{tb:table_balanced}, show that the proposed TDABC methods behave competitively concerning baseline classifiers (better than the average), even with broadly used classifiers such as SVM and RF. Specifically, when topology becomes complex like in the Swissroll case, where no hyperplane correctly identifies classes, TDABC shows better F1, ROC-AUC, PR-AUC, TNR, FPR, and GMEAN metrics than all reference methods. Type 1 errors or FPR is lower in our TDABC in comparison to the others classifiers in all databases except in Iris datasets and Swissroll, as we show in Figure~\ref{fig:all_metric}. \par

We also applied GMEAN to give a better measure of the combined growth of TNR and Recall. Our method shows in Figure~\ref{fig:all_metric},  higher values of GMEAN on datasets with a high imbalanced ratio like NORM and Sphere and in balanced datasets with non-linearly separable classes like Circles and Moon. In general, in each dataset, our method surpasses at least one of the baseline methods in GMEAN. The TNR measurement results in Figure~\ref{fig:all_metric} also show that our TDABC behaves well because our method differentiated negative instances similarly to RF and SVM and better than KNN and WKNN. Again, in balanced datasets with overlapped classes, our method overcame the others except for SVM in Circles, but in Moon, which has more entanglement between classes and more elements, our method was better. \par 

The Circles and Moon datasets are balanced and have very entangled classes due to the noise factor, making the classification more challenging. In these datasets, k-NN ($F1=0.449$ and $F1=0.445$) and wk-NN ($F1=0.480$ and $F1=0.463$) behave poorly (lower than average). This behavior is related to the fixed value of k and the assumption that each data point is equally relevant. Even though wk-NN imposes a local data point weight based on distances, it is not enough with highly entangled classes, as our results show. The TDABC methods are capable of dealing with the entanglement challenge through a disambiguation factor based on filtration values ($\xi_{\mathcal{K}}$). The Iris dataset is a simple case where practically all methods perform well. \par

Regarding the proposed TDABC method, we highlight that building a filtered simplicial complex to model high order data relationships enrich analysis but imposes the challenge to extract a meaningful collection of connections in a subcomplex. To perform this extraction the proposed three selection functions take advantage of the ability of persistent homology to unravel topological patterns hidden in data.

Specifically the link and star operators are used to recover variable-sized and multi-dimensional neighborhoods on the selected subcomplex, in contrast to fixed k-neighborhoods of points provided by KNN and WKNN. Furthermore for classification problems we proposed filtration values as local weights to give importance to each label contribution encoding the entire filtration history, giving less importance to outlier values and high to dense regions. 

As shown, the extraction and local weights combination effectively dealt with imbalanced datasets and with overlapped classes in balance on the other hand. 



\section{Conclusion}\label{scc:conclusions}

This paper presents a TDA based approach to classify imbalanced and balanced multi-class datasets. The proposed method is capable to classify multi-class data avoiding to solve multiple binary classification problems. The proposed classifier is a novel application of TDA for classification tasks without any additional ML method, in contrast with common TDA usages as a topological descriptor provider for ML. To our knowledge, this is the first study that proposes this approach for classification. Three methods were presented, TDABC-A, TDABC-R, TDABC-M, to address the binary and multi-imbalance data classification problems without any re-sampling preprocessing. Overall our methods behave better in average than the baseline methods in overlapped and minority classes.

Our work shows that PH plays a key role on selecting a sub-complex which approximate well enough data topology, through the {\it MaxInt}, {\it AvgInt} or {\it RandInt} selection functions. Moreover, link and star operators provide dynamic neighborhoods for classification. This work presents an application for the filtration values as the inverse weight to measure each simplex label contribution. These weights have several properties to deal with overlapped classes such as distance encoding operators, indirect local outlier factors, and density estimators. As a result, the labelling function depends on the filtration full history of the filtered simplicial complex and is encapsulated within the persistent diagrams at various dimensions.

Further, the labeling function also provides the probability of belonging to each class allowing, for instance, the identification and correction of mislabeled points in future work.

\section*{Supplementary information}




The source code used in this article is available \url{https://github.com/rolan2kn/TDABC-4-ADAC.git} in Github.

\section*{Acknowledgments}



This research work was supported by the National Agency for Research and Development of Chile (ANID), with grants ANID 2018/BECA DOCTORADO NACIONAL-21181978, FONDECYT 1211484, 1221696, ICN09\_015, and PIA ACT192015. Beca postdoctoral CONACYT (Mexico) also supports this work. The first author would like to thank professor Jos\'e Carlos G\'omez-Larra\~naga from CIMAT, Mexico, for his insightful discussions.

input(template.bbl)
\printbibliography
\end{document}